\documentclass[%
 reprint,
 amsmath,amssymb,
 aps,
]{revtex4-2}

\usepackage{graphicx}
\usepackage{dcolumn}
\usepackage{bm}

\newcommand{\R}{\mathbb{R}}

\newcommand{\tr}{{\rm tr}}
\newcommand{\col}{{\rm col}}

\newcommand{\T}{{\rm T}}
\newcommand{\TT}{{\mathbb{T}}}

\usepackage{booktabs,tabularx}
\renewcommand{\arraystretch}{1}

\newtheorem{theorem}{Theorem}
\newtheorem{proof}{Proof}

\begin{document}

\allowdisplaybreaks

\preprint{APS/123-QED}

\title{Training oscillator Ising machines to assign the dynamic stability\\ of their equilibrium points}
\thanks{Y. Cheng and Z. Lin are with Charles L. Brown Department of Electrical and Computer Engineering, University of Virginia, Charlottesville, Virginia 22904, USA.}

\author{Yi Cheng}
 \email{zss7gw@virginia.edu}
\author{Zongli Lin}%
 \email{zl5y@virginia.edu}
\affiliation{%
 Charles L. Brown Department of Electrical and Computer Engineering, University of Virginia, Charlottesville, Virginia 22904, USA.
}%





\begin{abstract}
We propose a neural network model, which, with appropriate assignment of the stability of its equilibrium points (EPs), achieves Hopfield-like associative memory. The oscillator Ising machine (OIM) is an ideal candidates for such a model, as all its $0/\pi$ binary EPs are structurally stable with their dynamic stability tunable by the coupling weights. Traditional Hopfield-based models store the desired patterns by designing the coupling weights between neurons. The design of coupling weights should simultaneously take into account both the existence and the dynamic stability of the EPs for the storage of the desired patterns.
For OIMs, since all $0/\pi$ binary EPs are structurally stable, the design of the coupling weights needs only to focus on assigning appropriate stability for the $0/\pi$ binary EPs according to the  desired patterns. In this paper, we establish a connection between the stability and the Hamiltonian energy of EPs for OIMs, and, based on this connection, provide a Hamiltonian-Regularized Eigenvalue Contrastive Method (HRECM) to train the coupling weights of OIMs for assigning appropriate stability to their EPs. Finally, numerical experiments are performed to validate the effectiveness of the proposed method.

\end{abstract}

\maketitle


\section{\label{sec:level1}Introduction}

Associative memory networks have been widely studied for decades and continue to exert a profound influence in fields such as pattern recognition, artificial intelligence, and cognitive science. The Hopfield networks, as classical associative memory networks, were first conceptualized by William A. Little \cite{LITTLE1974101} and formally introduced by John J. Hopfield \cite{hopfield1982neural}. An essential problem for associative memory models is how they store the desired patterns. In \cite{hopfield1982neural}, Hopfield introduced the Hebbian rule, which was shown to result in an capacity of $0.138N$ \cite{amit1985storing}, where $N$ is the size of the network. The pseudo-inverse rule was then proposed that would result in a higher capacity of $N$, at the cost of higher computational consumption \cite{kanter1987associative}. Later, Amos J. Storkey proposed a rule whose capacity and computation consumption are intermediate between the Hebbian and the pseudo-inverse rules \cite{storkey1999basins}. Subsequently, various learning rules were studied for improving the performance of Hopfield networks \cite{demircigil2017model,ramsauer2020hopfield,wu2024uniform,millidge2022universal}. Other classical associative memory networks, named the brain-state-in-a-box (BSB) neural networks, have also gained attention because of their continuous values, box-constrained dynamics and parallel synchronous updates that guarantee strict energy minimization and robust convergence \cite{golden1986brain}. Various coupling rules have been proposed to improve the performance of BSB neural networks for associative memory \cite{perfetti1995synthesis,schultz1993collective,hui1992dynamical,park2010optimal}.

These coupling rules possess a common feature, that is, they take both the existence and the stability of the equilibrium points (EPs) into consideration simultaneously. In other words, the resulting coupling weights should first guarantee that desired patterns exist in the network as EPs and then ensure that these EPs are asymptotically stable. 
As a result, for a fixed-size network with an increasing number of desired patterns, there is a competition between meeting the existence requirement and that of dynamical stability of the desired EPs, as has been observed in the literature \cite{amit1987statistical, carpenter1991pattern,folli2017maximum,fusi2021memory}. The underlying reason of such competition is that the binary EPs of those associative memory networks are structurally unstable with respect to coupling weights and any variation of the coupling weights in an effort to ensure stability of the EPs could cause some of those EPs to disappear. To circumvent this competition, one may construct a network whose binary EPs are structurally stable and remain unchanged as the coupling weights are varied to achieve the dynamical stability of the EPs.

OIMs are exactly such appealing networks. In \cite{cheng2024control}, it has been shown that all $0/\pi$ binary EPs are structurally stable with respect to the coupling weights but their dynamical stability is determined by the coupling weights. Specifically, a binary EP $\theta^\star$ of an OIM is asymptotically stable if and only if the corresponding Jacobian matrix $A(\theta^\star)$ is negative definite, where $A(\theta^\star)$ is constructed from $\theta^\star$ and the coupling weights. To realize Hopfield-like associative memory in an OIM, our objective is to design coupling weights so that only the desired patterns are asymptotically stable EPs (that is, the corresponding $A(\theta^\star)$'s are negative definite), while all others are unstable (that is, the corresponding $A(\theta^\star)$'s are nonnegative definite). In other words, the ideal coupling weights should ensure that the largest eigenvalue of $A(\theta^\star)$ for a desired EP is negative and is positive for a non-desired EP. If we regard the largest eigenvalue as ``energy", the realization of our objective can be achieved by Boltzmann machine learning (BML), in which the coupling weights are trained so that the observed data have less energy than the others.

However, BML will incur an excessive computational load if the largest eigenvalue is directly taken as energy. In BML, Gibbs sampling is used to draw samples from the model distribution for contrastive divergence algorithm \cite{hinton2002training}. In Gibbs sampling, the $k$-th neuron's state $s_k$ updates according to the probability
$p(s_k=1)=1/(1+\exp\{-\Delta E(s)\})$,
which involves the calculation of $\Delta E(s)$ that results from flipping $s_k$ from $-1$ to $+1$. If the largest eigenvalue is used as the energy, every sampling step requires computing that eigenvalue - an $O(N^3)$ operation - resulting in markedly reduced efficiency as the network size increases.

Fortunately, we observe that the largest eigenvalue of $A(\theta^\star)$ is roughly linearly correlated with the corresponding Hamiltonian energy of $\theta^\star$, which provides a qualitative characterization in Theorem
\ref{Thembounded}. Then, Gibbs sampling, with Hamiltonian as the energy, can be applied to draw samples whose largest eigenvalues are relatively smaller. Therefore, in order to train the coupling weights by BML such that only the EPs for the desired patterns are asymptotically stable, our objective function is an eigenvalue related term. We adopt Gibbs sampling with the Hamiltonian as energy in our training process. Our experiment results show the effectiveness of this proposed approach. To improve the performance, we have also added a Hamiltonian regularized term to our objective function. As a result, we refer to this training method as Hamiltonian-Regularized Eigenvalue Contrastive Method (HRECM).

The remainder of the paper is organized as follows. Section \ref{sec:oim-properties} introduces some basic dynamic properties of OIMs and the difference between OIMs and traditional associative memory networks. Section \ref{sec:training} presents the framework of our training method. Section \ref{sec:conclusions} presents numerical experiments to validate the effectiveness of the training. Section \ref{sec:conclusions} concludes the paper.

{\emph{Notation}}: Throughout the paper, we will use standard notation.  We use $\mathbb{R}$ to denote the set of all real numbers. The torus is the set $\TT = (-\pi,\pi]$. Given vectors $x_1,x_2, \ldots, x_N$, ${\rm col}\{x_1,x_2,\ldots,x_N\}=\left[x_1^{\rm T}\;x_2^{\T}\ldots x_N^{\T}\right]^{\T}$.
We use $\{a,b\}^N$ to denote the set of vectors $\{\col\{x_1,x_2,\ldots,x_N\}: x_i \in \{a,b\},i=1,2,\ldots,N\}$.
Consider a vector $s\in \R^N$, we use $s[s_k: =a]$ to denote the vector whose $k$-th component is $a$ and whose other components are the same as in $s$.

\section{The properties of OIMs}
\label{sec:oim-properties}

\subsection{The stability of $0/\pi$ binary EPs}
The dynamics of an OIM of $N$ oscillators is given by
\begin{equation}\label{dyOIM}
 \dot \theta_{i} = -K\sum_{j=1}^N J_{ij}\sin(\theta_{i}-\theta_{j})-K_{\rm s}\sin(2\theta_{i}),
\end{equation}
where $\theta_{i} \in \TT$ is the $i$-th phase of the vector $\theta=\col\{\theta_{1},\theta_{2},\ldots,\theta_{N}\}$, $K,K_{\rm s} >0$ are tunable parameters, and $J_{ij} \in \R$ is the coupling weight between the $i$-th and $j$-th oscillators, with $J_{ij}=J_{ji}$ and $J_{ii}=0$. It can be verified that the energy function
\begin{equation}\label{enrgyOIM}
E(\theta)=-K\!\sum_{i=1}^N\!\sum_{j=1}^N\!J_{ij}\cos(\theta_{i}-\theta_{j})\!-\!K_{\rm s}\sum_{i=1}^N\!\cos(2\theta_{i})
\end{equation}
is decreasing along the trajectory of \eqref{dyOIM}. In addition, when $\theta \in \{0,\pi\}^N$, $E(\theta)$ coincides with the Ising Hamiltonian
\[H(s) = -\sum_{i<j}^NJ_{ij}s_{i}s_{j}\]
scaled by a factor of $2K$ on $s \in \{+1,-1\}^N$
(with $s_{i}=+1$ if $\theta_{i}=0, s_{i}=-1$ if $\theta_{i}=\pi$),  up to an additive constant $-NK_{\rm s}$. It can be verified that any vector $\theta^\star \in \{0,\pi\}^N$ or $\theta^\star \in \{-\frac{\pi}{2},\frac{\pi}{2}\}^N$ is a structurally stable EP of \eqref{dyOIM}, referred to as the $0/\pi$ and $-\frac{\pi}{2}/\frac{\pi}{2}$ EPs, respectively. In \cite{cheng2024control}, it has been proven that $-\frac{\pi}{2}/\frac{\pi}{2}$ EPs are dynamically unstable and that a $0/\pi$ EP $\theta^\star$ is asymptotically stable if and only if $A(\theta^\star) = KD(\theta^\star)-2K_{\rm s}I_N$ is negative definite, that is,
\begin{equation}\label{Stabilitycon}
    \frac{K_{\rm s}}{K} >\frac{\lambda_N(D(\theta^\star))}{2},
\end{equation}
where $\lambda_N(D(\theta^\star))$ is the largest eigenvalue of $D(\theta^\star)$, and $D(\theta^\star)$ is a symmetric matrix, with diagonal element $D_{ii}(\theta^\star) = -\sum_{j=1}^NJ_{ij}\cos(\theta_{i}^\star-\theta_{j}^\star)$ and off-diagonal element $D_{ij}(\theta^\star) = J_{ij}\cos(\theta_{i}^\star-\theta_{j}^\star)$. Although there are other EPs (see \cite{cheng2024control}), we focus on $0/\pi$ EPs in this work.

\subsection{The difference between OIMs and traditional associative memory networks}

Capacity is an important metric to measure the performance of the coupling design for traditional associative memory. On the contrary, for OIMs, capacity is not a good metric because all $0/\pi$ binary EPs can be asymptotically stable, regardless of the coupling weights, if $K_{\rm s}/K$ is large enough. To see this, an obvious example is that for OIMs and Hopfield networks, both on an unweighted all-to-all connected graph, OIMs can effectively store $2^N$ binary patterns, while Hopfield networks can only store two binary patterns, $\col\{+1,+1,\ldots,+1\}$ and $\col\{-1,-1,\ldots,-1\}$. See Fig. \ref{Hop_OIM.pdf} for an illustration.

\begin{figure}[htb]
    \centering
    \includegraphics[width=3in]{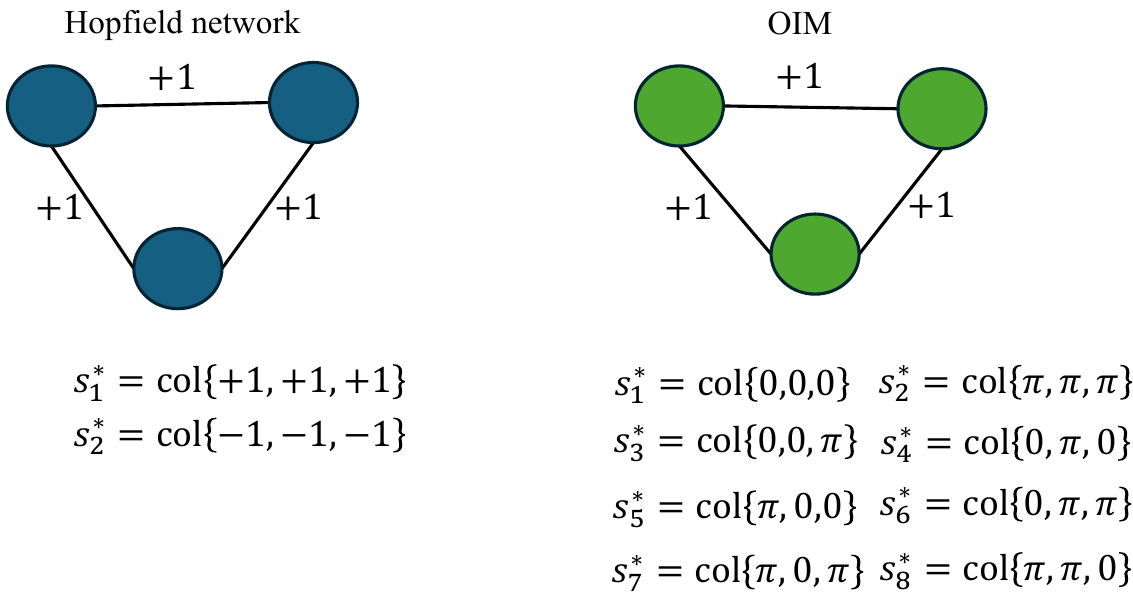}
    \caption{The difference of capacity between Hopfield networks and OIMs on an unweighted graph.}
    \label{Hop_OIM.pdf}
\end{figure}

However, when the number of desired patterns is very small, making the desired patterns asymptotically stable EPs by large enough $K_{\rm s}/K$ could introduce many asymptotically stable $0/\pi$ spurious EPs if the coupling weights are not properly designed. Therefore, for OIMs to realize Hopfield-like associative memory, the design of their coupling weights is essential. We note that the coupling designs for OIMs cannot eliminate any $0/\pi$ EPs as they are structurally stable with respect to coupling weights. On the other hand, according to \eqref{Stabilitycon}, the coupling weights can affect the dynamic stability of these EPs. As a result, we may train the coupling weights so that only desired EPs are asymptotically stable to store the desired patterns and other EPs are unstable. Thus, an OIM will not reliably retrieve patterns if they were stored in those unstable EPs, reducing, in effect, the number of spurious patterns.

Since capacity is not a good metric for OIMs, a new metric should be defined to measure the performance of the OIMs as associative memory. Such a new metric should be related to the number of unwanted EPs that are asymptotically stable. We will discuss this metric in Section \ref{sec:experiments}.


\section{Training OIMs for assigning stability of $0/\pi$ EPS}
\label{sec:training}

\subsection{Gibbs sampling}
Given a network system, equipped with an energy $E(s)$ that is related to the system state $s$, the samples obtained from Gibbs sampling tend to have lower energy. Specifically, for an associative memory network of $N$ neurons, with the $k$-th neuron's state $s_k \in \{+1,-1\}$, if $s_k$ sequentially updates to $+1$ with the probability
\begin{equation}\label{Gibbs}
p(s_k=+1) = \frac{1}{1+\exp\{-\Delta E(s)\}},
\end{equation}
where $\Delta E(s) = E(s[s_k:=-1])-E(s[s_k:=+1])$, then the state $s\in \{+1,-1\}^N$ will obey the probability distribution
\[p(s) = \frac{\exp\{-E(s)\}}{Z},\]
where
\[Z = \sum_{s \in S}\exp\{-E(s)\}\]
and $S$ is the set of all binary states. The update law \eqref{Gibbs} is the Gibbs sampling.

In spin systems, the energy is often characterized by the Hamiltonian, that is,  $E(s)=H(s) = -\sum_{i<j}^NJ_{ij}s_is_j$. In this case, \eqref{Gibbs} takes the form
\begin{equation}\label{HGibbs}
    p(s_k=+1) = \frac{1}{1+\exp\{-2\sum_{j}J_{kj}s_j\}},
\end{equation}
which is used by Hinton for BML \cite{ackley1985learning,hinton2002training}. To distinguish \eqref{HGibbs} from \eqref{Gibbs}, we refer to \eqref{HGibbs} as Hamiltonian Gibbs sampling.

\subsection{The stability and Hamiltonian  of $0/\pi$ EPs}

When $\theta_{i}^\star \in \{0,\pi\}$ and $s_{i}\in \{+1,-1\}$, it can be verified that $s_{i}s_{j} = \cos(\theta_{i}^\star-\theta_{j}^\star)$ (with $s_{i}=+1$ if $\theta_{i}^\star=0$, $s_{i}=-1$ if $\theta_{i}^\star=\pi$). Thus, the matrix
\[A(\theta^\star) := A(s) =KD(s)-2K_{\rm s}I_N,\]
where $D(s)$ is a symmetric matrix, with diagonal element $D_{ii}(s)=-\sum_{j=1}^NJ_{ij}s_{i}s_{j}$ and off-diagonal element $D_{ij}(s)=J_{ij}s_{i}s_{j}$. Let $S$ be the set of all binary patterns and $S_{\rm d} = \{s^\star_d \in \{+1,-1\}^N:d=1,2,\ldots,m\}$ be the set of $m$ desired patterns.  An OIM is said to achieve perfect associative memory if
\begin{equation}\label{perfectAM}
\max_{s^\star_{d} \in S_{\rm d}}\lambda_N(D(s^\star_d))<\min_{s\in S \setminus S_{\rm d}}\lambda_N(D(s)).
\end{equation}
Inequality \eqref{perfectAM} is said to be perfect for associative memory because  a choice of  $\frac{K_{\rm s}}{K}$ such that \[\max_{s^\star_d \in S_{\rm d}}\lambda_N(D(s^\star_d))< \frac{K_{\rm s}}{K}<\min_{s\in S \setminus S_{\rm d}}\lambda_N(D(s))\]
ensures that, by \eqref{Stabilitycon},
the $m$ desired $0/\pi$ EPs are asymptotically stable and all other $0/\pi$ EPs are unstable. Our extensive experiments indicate that there is an approximate positive correlation between $\lambda_N(D(s))$ and $H(s)$, and the next theorem specifically characterizes this connection.

\begin{theorem}\label{Thembounded}
Consider an OIM of $N$ oscillators with given coupling weights $J_{ij}$. Let $\lambda_N(D(s))$ be the largest eigenvalue of the matrix $D(s)$ and $H(s)$ be the Hamiltonian. Then, for any $s \in \{+1,-1\}^N$,
\begin{eqnarray}\label{bound}
&&\max\left\{0,\tfrac{2H(s)}{N}\right\}
\le\ \lambda_N\bigl(D(s)\bigr)
  \notag\\
\le &&\frac{2+2\sqrt{(\kappa N-1)(N-1)}}{N}\,H(s)+ d.
\end{eqnarray}
where
\begin{eqnarray*}
\kappa &=& \max_{s \in S}\left\{\frac{{\rm tr}(\hat D(s)^2)}{{\rm tr}^2(\hat D(s))}\right\},\\
d &=& c\sqrt{(\kappa N-1)(N-1)},\\
\hat D(s) &=& D(s)+cI_N,
\end{eqnarray*}
with $c$ being a constant such that
\[c>\frac{-2H(s)}{N},\;s\in S.\]
\end{theorem}

\begin{proof}
See Appendix.
\end{proof}

Theorem \ref{Thembounded} shows that the stability of a $0/\pi$ binary EP of an OIM is roughly related to the EP's Hamiltonian, that is, the largest eigenvalue of $D(s)$ corresponding to a $0/\pi$ binary EP with a lower Hamiltonian tends to be smaller than that of $D(s)$ corresponding to a $0/\pi$ binary EP with a higher Hamiltonian. Thus, a sample $s$ drawn by Hamiltonian Gibbs sampling tends to have smaller $\lambda_N(D(s))$ compared to the other states. This property will be explored in the next subsection.

\subsection{Hamiltonian-regularized eigenvalue contrastive method}

Since the desired coupling weights should be such that \eqref{perfectAM} is satisfied, the coupling weights can be trained by maximizing the objective function $L(s^\star_d)$, that is,
\begin{eqnarray}
\max_{J_{ij}} L(s^\star_d)
&=& \min_{s\in S\setminus S_{\mathrm{d}}} \lambda_N\bigl(D(s)\bigr)
  -\max_{s^\star_d\in S_{\mathrm{d}}} \lambda_N\bigl(D(s^\star_d)\bigr)\nonumber\\
&&+\;\alpha P(s^\star_d),\label{objective}
\end{eqnarray}
where we refer to the first two terms as the eigenvalue contrastive term and the last term is the regularized term. Also,
\[P(s^\star_d) = \sum_{s^\star_d\in S_{\mathrm{d}}}
  \frac{\exp\left\{-H(s^\star_d)\right\}}{Z},\]
where $H(s^\star_d)$ is the Hamiltonian,
\[Z = \sum_{s \in S} \exp\{-H(s)\},\]
and $\alpha>0$ is the regularized factor.

Maximizing $L(s^\star_d)$ can be achieved by computing the gradient as the updating direction. At each step, with given coupling weights $J_{ij}$'s, let
\begin{eqnarray*}
S^{m} &=& \arg_{s\in S\setminus S_{\rm d}}\min \lambda_N(D(s)),\\
S^{M} &=& \arg_{s^\star_d\in S_{\rm d}} \max \lambda_N(D(s^\star_d)),
\end{eqnarray*}
$s^{\min} \in S^m$ and $s^{\max} \in S^M$ be two states randomly chosen from $S^m$ and $S^M$, respectively, and
\begin{eqnarray*}
v^{\min} &=& \col\{v^{\min}_1, v^{\min}_2,\ldots, v^{\min}_N\},\\
v^{\max} &=& \col\{v^{\max}_1,v^{\max}_2,\ldots, v^{\max}_N\}
\end{eqnarray*}
be the normalized eigenvectors associated with $\lambda_N(D(s^{\min}))$ and $\lambda_N(D(s^{\max}))$, respectively. Since $\lambda_N(D(s))$ is Lipchitz continuous, but not necessarily differentiable with respective to the coupling weights $J_{ij}$, we adopt the generalized (Clark) gradient \cite{overton1992large}. If $\lambda_N(D(s))$ has multiplicity $p$, then
\begin{eqnarray*}
 &&\frac{\partial \lambda_N(D(s))}{\partial J_{ij}} = \left\{\rho_{ij}\in \R: \rho_{ij}= \left<\!Q_1(s)UQ_1^\T(s),\frac{\partial D(s)}{\partial J_{ij}}\right>,\right. \\
 &&\quad\left.\text{for some } U\in \R^{p\times p},U\ge 0, {\rm tr}(U)=1\rule{0cm}{.45cm}\right\},
\end{eqnarray*}
where the columns of $Q_1(s)\in \R^{n \times p}$ form an orthonormal set of (generalized) eigenvectors of $\lambda_N(D(s))$, with the first column as an eigenvector, ${\rm tr}(U)$ is the trace of the matrix $U$, and $\left<Q_1(s)UQ_1^\T(s),\frac{\partial D(s)}{\partial J_{ij}}\right>$ denotes the Frobenius inner product of two matrices. Note that the Clark gradient is set-valued, and thus we represent (approximate) the gradient by fixing
\[U = \left[\begin{array}{cccc}
    1 & 0 & \cdots & 0 \\
    0 & 0 & \cdots & 0 \\
    \vdots & \vdots & \ddots & \vdots\\
    0 & 0 & \cdots& 0
\end{array} \right].\]
Therefore, the gradient can be written as
\begin{eqnarray}\label{eigngrad}
    \frac{\partial \lambda_N(D(s))}{\partial J_{ij}} &\approx& \left<Q_1(s)UQ_1^\T(s),\frac{\partial D(s)}{\partial J_{ij}}\right> \nonumber \\
    &=& {\rm tr}\left(v(s)v^\T(s)\frac{\partial D(s)}{\partial J_{ij}}\right) \nonumber\\
    &=& {\rm tr}\left(v^\T(s)\frac{\partial D(s)}{\partial J_{ij}}v(s)\right) \nonumber\\
    &=& v^\T(s)\frac{\partial D(s)}{\partial J_{ij}}v(s),
\end{eqnarray}
where $v(s)$ is an eigenvector associated with $\lambda_N(D(s))$. The equality holds if and only if $\lambda_N(D(s))$ has multiplicity one. Similarly, since $L(s_d^\star)$ is not necessarily differentiable, the Clark gradient is applied. Specifically,
\begin{eqnarray}
    \frac{\partial L(s_d^\star)}{\partial J_{ij}}& = &{\rm conv}\left\{\frac{\partial \lambda_N(D(s))}{\partial J_{ij}}:s \in S^m\right\} \nonumber \\
    &&\!\!\!\!- {\rm conv}\left\{\frac{\partial \lambda_N(D(s))}{\partial J_{ij}}:s \in S^M\right\} \nonumber \\
    &&+ \alpha \frac{\partial P(s_d^\star)}{\partial J_{ij}}  \nonumber \\
    &\approx& \frac{\partial \lambda_N(D(s^{\min}))}{\partial J_{ij}} - \frac{\partial \lambda_N(D(s^{\max}))}{\partial J_{ij}} \nonumber \\
    &&+ \alpha \frac{\partial P(s_d^\star)}{\partial J_{ij}}, \label{gradL1}
\end{eqnarray}
where ${\rm conv}$ denotes the convex hall operator. The equality holds in \eqref{gradL1} if and only if $S^m$ and $S^M$ have a single element, respectively.
Then, by \eqref{eigngrad} and \eqref{gradL1}, the gradient of $L(s_d^\star)$ is calculated as
\begin{eqnarray}\label{Grad}
    \Delta J_{ij} &=& \frac{\partial L(s^\star_d)}{\partial J_{ij}} \nonumber \\
    &\!\!\!\!\approx &\!\! \frac{\partial \lambda_N(D(s^{\min}))}{\partial J_{ij}} - \frac{\partial \lambda_N(D(s^{\max}))}{\partial J_{ij}} + \alpha\frac{\partial P(s^\star_d)}{\partial J_{ij}}\nonumber \\
    & \!\!\!\! \approx &\!\! (v^{\min})^\T \frac{\partial D(s^{\min})}{\partial J_{ij}}v^{\min} - (v^{\max})^\T \frac{\partial D(s^{\max})}{\partial J_{ij}}v^{\max}  \nonumber \\ && + \alpha\frac{\partial P(s^\star_d)}{\partial J_{ij}}\nonumber \\
    & \!\!\!\!= & \!\!\!- (v^{\min}_i\!\!-\!v^{\min}_j)^2s^{\min}_is^{\min}_j + (v^{\max}_i\!\!-\!v^{\max}_j)^2s^{\max}_is^{\max}_j \nonumber \\
    &&+\alpha \frac{\partial P(s^\star_d)}{\partial J_{ij}}.
\end{eqnarray}
Although $s^{\min}$ can be obtained by searching in the finite space $S\setminus S_{\rm d}$, it is computationally expensive for a large value of $N$. To address the issue, we adopt the Hamiltonian Gibss sampling to approximate $s^{\min}$. Since Theorem \ref{Thembounded} establishes that $\lambda_N(D(s))$ tends to be smaller if the Hamiltonian $H(s)$ is smaller, $s^{\min}$ can be approximated by $\hat s^{\min}$ drawn from Hamiltonian Gibbs sampling. Let $S^-= \{s^-_k:k=1,2,\ldots,n\}$ be the set of $n$ samples from Hamiltonian Gibbs sampling,
\[\hat s^{\min} = \arg_{s\in S^-} \min \lambda_N(D(s)),\]
and $\hat v^{\min}$ be the normalized eigenvector associated with $\lambda_N(D(\hat s^{\min}))$. Then, the gradient \eqref{Grad}
can be approximated by
\begin{eqnarray}\label{ApproxGrad}
\Delta J_{ij} &\!\approx\!& - (\hat v^{\min}_i\!\!-\!\hat v^{\min}_j)^2\hat s^{\min}_i \hat s^{\min}_j + (v^{\max}_i\!\!-\!v^{\max}_j)^2s^{\max}_is^{\max}_j \nonumber \\
    &&+\alpha \left( \left<s_{i}s_{j}\right>_{S_{\rm d}} -<s_{i}s_{j}>_{S^-}\right),
\end{eqnarray}
where $<s_{i}s_{j}>_{S_{\rm d}}$ denotes the expectation of $s_{i}s_{j}$ over all $s \in S_{\rm d}$, and  $<s_{i}s_{j}>_{S^-}$ denotes the expectation of $s_{i}s_{j}$ over all $s \in S^-$.  Based on this approximate, we propose the following algorithm for training $J_{ij}$.

\begin{table}[htbp!]
  \centering
  \setlength{\tabcolsep}{4pt}
  \renewcommand{\arraystretch}{0.9}
  \begin{tabularx}{\columnwidth}{@{}l l X@{}}
    \toprule
    \multicolumn{3}{@{}l}{\textbf{Algorithm 1:} Training Procedure for OIMs}\\
    \midrule
    \textbf{Step 1}   &              & \!\!\!\!\!\!\!\!\!\!\!\!\!\!\!\!\!\!\!\!\!\!\!\!\!\!\!Set $\alpha$, $\epsilon$, desired patterns $\{s^\star_d\}_{d=1}^m$, and initialize $J_{ij}$.\\
    \addlinespace[0.3ex]
    \textbf{Step 2}   &              & \!\!\!\!\!\!\!\!\!\!\!\!\!\!\!\!\!\!\!\!\!\!\!\!\!\!\!Obtain samples $\{s_k^-\}_{k=1}^n$; find $\hat s^{\min}$ and $s^{\max}$:\\
    \addlinespace[0.3ex]
                      & Step 2.1     & Sample $\{s_k^-\}_{k=1}^n$ according to Eq.\eqref{HGibbs}.\\
                      & Step 2.2     & $\hat s^{\min}=\arg\min\lambda_N(D(s_k^-))$, compute its normalized eigenvector $\hat v^{\min}$.\\
                      & Step 2.3     & $s^{\max}=\arg\max\lambda_N(D(s^\star_d))$, compute its normalized eigenvector $v^{\max}$.\\
    \addlinespace[0.3ex]
    \textbf{Step 3}   &              &\!\!\!\!\!\!\!\!\!\!\!\!\!\!\!\!\!\!\!\!\!\!\!\!\!\!\!\!\! Update $J_{ij}$'s by the patterns and samples:\\
    \addlinespace[0.3ex]
                      & Step 3.1     & Compute the first term in \eqref{ApproxGrad} by $\hat s^{\min}, \hat v^{\min}$.\\
                      & Step 3.2     & Compute the second term in \eqref{ApproxGrad} by $ s^{\max},  v^{\max}$.\\
                      & Step 3.3     & Compute the last term by $\alpha$, $\{s^\star_d\}_{d=1}^m$ and $\{s_k^-\}_{k=1}^n$.\\
                      & Step 3.4     & $J_{ij}(k+1) = J_{ij}(k) + \epsilon \Delta J_{ij}$ \\
    \textbf{Step 4}   &              & \!\!\!\!\!\!\!\!\!\!\!\!\!\!\!\!\!\!\!\!\!\!\!\!\!\!\!Repeat Steps 2 and 3 until stopping criterion is met.\\
    \bottomrule
  \end{tabularx}
\end{table}


\section{Numerical experiments}\label{sec:experiments}

\textbf{Experiment 1.} In this experiment, we show the effectiveness of our training method. We consider an OIM of $N=20$ oscillators and randomly choose $15$ desired patterns. In the Gibbs sampling stage, we draw $n=500$ samples.  For training parameters, we set the regularization factor $\alpha =0.8$, the learning rate $\epsilon=0.01$, and perform $T=1,100,300,2000$ training iterations. The $(H(s),\lambda_N(D(s))):=(E(\theta^\star),\lambda_{N}(D(\theta^\star)))$ plane is shown in Fig.~\ref{EX1.pdf} for the four training results. Note that every $0/\pi$ EP $\theta^\star$ and its negative $\theta^\star+\col\{\pi,\pi,\ldots,\pi\}$ have the same Hamiltonian and the corresponding eigenvalues of $D(s)$. Thus, on each plot, there are $2^{19} = 524,288$ points, each corresponding to two $0/\pi$ EPs. The red points are the desired patterns, and the blue points are spurious patterns. We observe that during the training process the largest eigenvalues corresponding to the desired patterns decrease and eventually become the smallest $15$ eigenvalues, separating themselves from all other $0/\pi$ EPs. As a result, $\frac{K_{\rm s}}{K}$ can be appropriately chosen so that all the desired $0/\pi$ EPs are asymptotically stable and all other $0/\pi$ EPs are unstable.

\begin{figure*}
    \centering
    \includegraphics[width=6.32in]{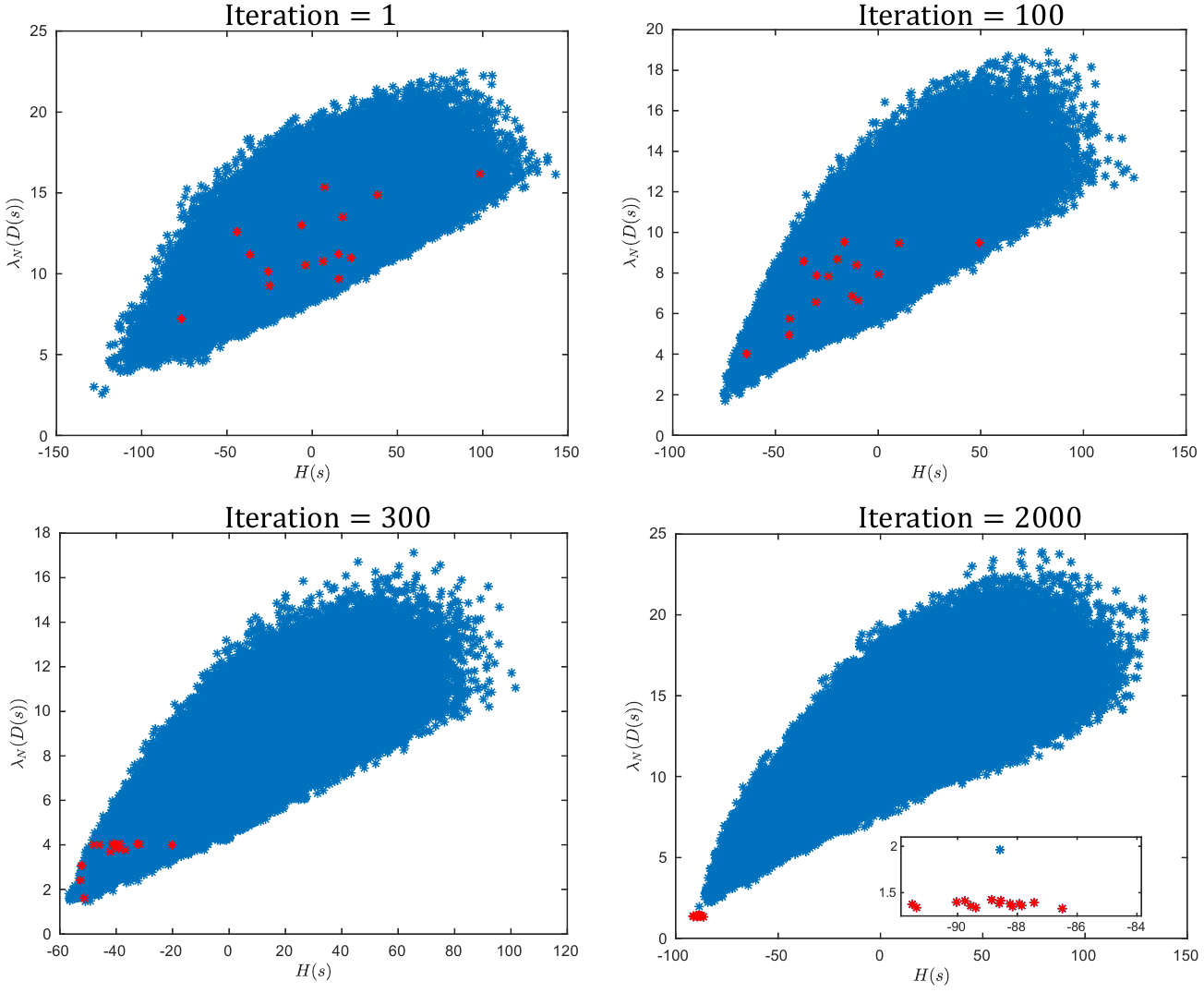}
    \caption{Evolution of the training process}
    \label{EX1.pdf}
\end{figure*}

\textbf{Experiment 2.} In this experiment, we compare the performance of our Algorithm 1 with and without regularization. We consider four OIMs of $N=10,15,20,25$ oscillators, respectively. For each OIM, we test Algorithm 1 with and without regularization ($\alpha =0$ and $\alpha = 1$), with the number $m$ of desired patterns varying from five to $30$  in increments of five. We define the spurious rate as $R_{\rm s} = \frac{a}{2^{N-1}}$ to be the metric for measuring the performance, where $a$ is the number of $0/\pi$ spurious EPs whose $\lambda_N(D(\theta^\star))$'s are less than that of at least one of the desired patterns. In computing $a$, each $0/\pi$ EP $\theta^\star$ and its negative $\theta^\star+{\rm col}\{\pi,\pi,\ldots,\pi\}$ are treated as a single entity and counted only once. Clearly, $ 0 \leq R_{\rm s} < 1$. Values of $R_{\rm s}$ closer to zero indicate better performance, and OIMs achieve perfect associative memory if $R_{\rm s} =0 $. For all cases, we fix all training parameters at the same values, the learning rate $\epsilon = 0.01$, the training iterations $T=5000$, the number of samples in Hamiltonian Gibbs sampling $n=500$, except for the regularization factor $\alpha=1$ for the regularized algorithm and $\alpha=0$ for the non-regularized algorithm. Fig.~\ref{ALgorithmcompare1.pdf} shows the evolution of $R_{\rm s}$.

We observe that values $R_{\rm s}$ of the regularized algorithm are smaller than or equal to those of the non-regularized algorithm. Specifically, they are almost the same when the number $m$ of the desired patterns is smaller than or equal to the size $N$ of the network, while the values $R_{\rm s}$ of the regularized algorithm are much smaller than those of the non-regularized algorithm when $m>N$.

\begin{figure}[h!]
    \centering
    \includegraphics[width=3.4in]{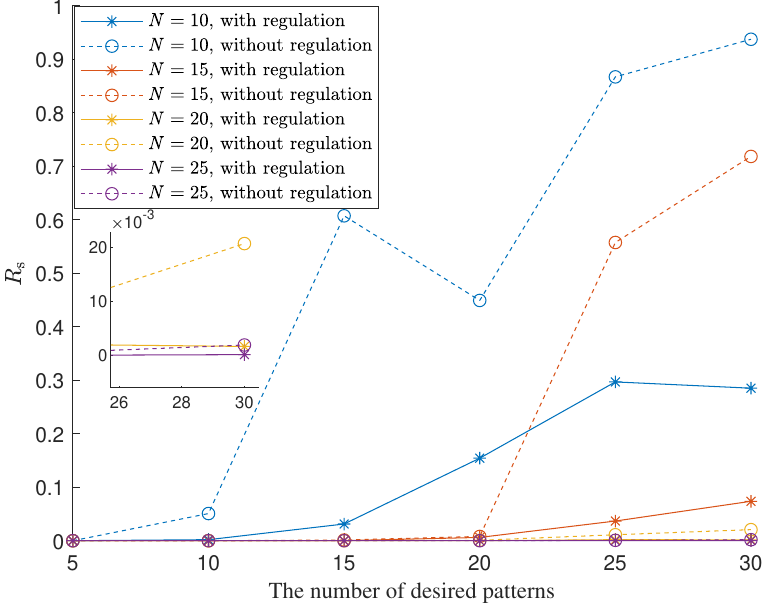}
    \caption{The comparison of the algorithm with and without regularization}
    \label{ALgorithmcompare1.pdf}
\end{figure}

\textbf{Experiment 3.} In this experiment, we consider the regularized algorithm for larger networks. The network size $N$ varies from $20$ to $70$ in increments of five. For each OIM, we test the algorithm with the number $m$ of desired patterns varying from $15$ to $70$ in increments of five. For a larger network, the spurious rate is difficult to determine as $a$ counts all $0/\pi$ spurious EPs. Theorem \ref{Thembounded} indicates that the largest eigenvalues in samples obtained from the Hamiltonian Gibbs sampling tend to be smaller among all $0/\pi$ EPs. Thus, instead of enumerating all $0/\pi$ spurious EPs, we sample them by Hamiltonian Gibbs sampling and then evaluate the spurious rate. Specifically, we revise the spurious rate as $R_{\rm s}= \frac{a}{b}$, where $b$ is the number of distinct $0/\pi$ EPs obtained by Hamiltonian Gibbs sampling, and $a$ is the number of $0/\pi$ spurious EPs in samples whose $\lambda_{N}(D(\theta^\star))$'s are less than those of at least one of the desired $0/\pi$ EPs. In the training process, we fix the training parameters, the learning rate $\epsilon=0.01$, the training iterations $T=5000$, the number of samples from Hamiltonian Gibbs sampling $n=500$, the regularization factor $\alpha =1$. After the end of the training, we send the resulting coupling weights to the Hamiltonian Gibbs sampling and sample $20000$ times. Note that $20000$ times of sampling generally do not return $20,000$ distinct EPs. Thus, $b<20,000$ in general. Clearly, $0\leq R_{\rm s}< 1$. The evolution of $R_{\rm s}$ is shown in Fig.~\ref{R_sVsPatterns_over_N_from_20_to_70.pdf}. We observe that for a fixed size $N$, the overall tendency of the $R_{\rm s}$ value is increasing as the number of the desired $0/\pi$ EPs increases, while for a fixed number of desired patterns, the $R_{\rm s}$ value is increasing as the size $N$ decreases.

\begin{figure}[h!]
    \centering
    \includegraphics[width=3.4in]{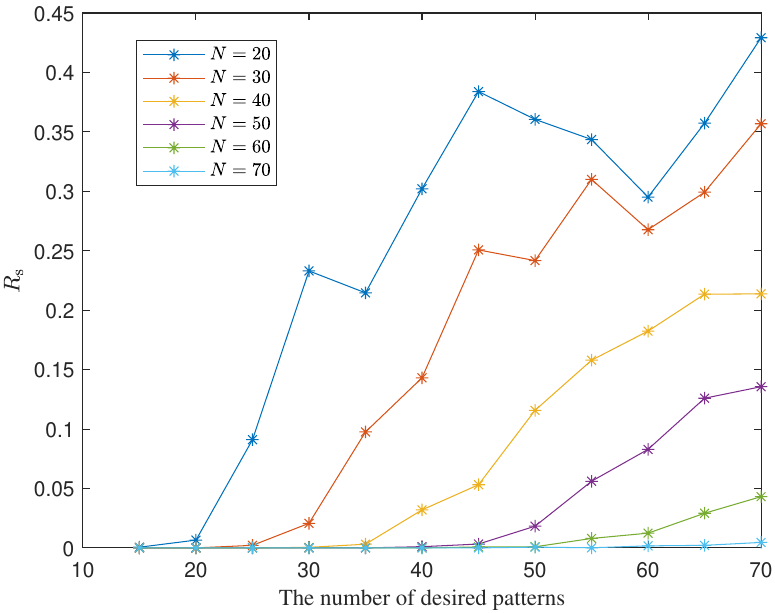}
    \caption{The evolution of $R_{\rm s}$ for larger OIMs: with regularization}
\label{R_sVsPatterns_over_N_from_20_to_70.pdf}
\end{figure}

\section{Conclusions}
\label{sec:conclusions}

In this paper, we proposed a neuron network model that, with appropriate assignment of the stability of its equilibrium points (EPs), can achieve Hopfield-like associative memory. We adopted OIMs as candidates for such a model. Because of the structural stability of $0/\pi$ binary EPs of OIMs, the training of their coupling weights needs only to focus on achieving dynamic stability of the desired $0/\pi$ EPs and thus can be efficient. We hope the proposed method would provide a novel perspective on the framework of associative memory research.


\begin{acknowledgments}
This material is based upon work supported by the National Science Foundation under grant no. 2328961.
\end{acknowledgments}

\appendix

\section{Proof of Theorem \ref{Thembounded}}

Since the sum of each row of $D(s)$ is zero and ${\rm tr}(D(s)) = 2H(s)$, the left inequality of \eqref{bound} is trivial.  Let $\hat D(s) = D(s) + cI_N$. By the definition of $c$, we have ${\rm tr}(\hat D(s))\ge0$ for all $s \in S$. It follows from Theorem 2.1 in \cite{wolkowicz1980bounds} that
\begin{eqnarray*}
  \lambda_N(\hat D(s)) &\leq & \frac{\tr(\hat D(s))}{N}
   \!+\!\sqrt{\frac{\tr(\hat D(s)^2)-\frac{\tr^2(\hat D(s))}{N}}{N}}\sqrt{N-1} \\
   &\leq& \frac{\tr(\hat D(s))}{N}
   \!+\!\sqrt{\frac{\kappa \tr^2(\hat D(s))-\frac{\tr^2(\hat D(s))}{N}}{N}}\!\sqrt{N-1} \\
   &=& \frac{1+\sqrt{(\kappa N-1)(N-1)}}{N}\tr(\hat D(s)).
\end{eqnarray*}
Since $\lambda_N(D(s))=\lambda_N(\hat D(s))-c$,
\begin{eqnarray*}
 \lambda_N(D(s)) &\leq& \frac{1+\sqrt{(\kappa N-1)(N-1)}}{N}\tr(\hat D(s)) -c \\
 &=&\frac{1+\sqrt{(\kappa N-1)(N-1)}}{N}\left(\tr(D(s))\!+\!Nc\right) -c  \\
 &=& \frac{1+\sqrt{(\kappa N-1)(N-1)}}{N}\tr(D(s))+d\\
 &=& \frac{2+2\sqrt{(\kappa N-1)(N-1)}}{N}H(s)+d.
\end{eqnarray*}








\bibliography{apssamp}

\end{document}